%% file: main.tex
\documentclass{article}

\usepackage[round]{natbib}
\renewcommand\cite{\citep}

\usepackage[
left=1in,
right=1in,
top=1in,
bottom=1in,
footskip=.25in]{geometry}

\usepackage[utf8]{inputenc} %
\usepackage[T1]{fontenc}    %
\usepackage{hyperref}       %
\usepackage{url}            %
\usepackage{booktabs}       %
\usepackage{amsfonts}       %
\usepackage{nicefrac}       %
\usepackage{microtype}      %
\usepackage{amsmath}
\usepackage{mathtools}
\usepackage{amsthm}
\usepackage{amssymb}
\usepackage{xparse}
\usepackage{xcolor}
\usepackage{graphicx}

\usepackage{soul}
\setstcolor{red}

\usepackage[linesnumbered,ruled,vlined]{algorithm2e}

\newcommand{\bra}[1]{\left( #1 \right)}
\newcommand{\brb}[1]{\left[ #1 \right]}
\newcommand{\brc}[1]{\left\{ #1 \right\}}
\newcommand{\nm}[1]{\left\lVert#1 \right\rVert}
\newcommand{\infnm}[1]{\nm{#1}_\infty}
\newcommand{\snm}[1]{\nm{#1}_2}
\newcommand{\abs}[1]{\left\lvert #1 \right\rvert}
\newcommand{\range}[1]{range\bra{#1}}

\newcommand{\Fc}{\mathcal{F}}
\newcommand{\Gc}{\mathcal{G}}
\newcommand{\Eb}{\mathbb{E}}
\newcommand{\Rb}{\mathbb{R}}
\newcommand{\Nb}{\mathbb{N}}
\newcommand{\Pb}{\mathbb{P}}

\newcommand{\opterr}{\tau_{\mathrm{opt}}}

\newcommand{\ferr}[1]{\tau_{#1}}
\newcommand{\ferrlong}[1]{\tau_{\Fc,\Gc}\bra{#1}}

\newcommand{\tverr}[1]{\epsilon_{\mathrm{TV}}\bra{#1}}

\newcommand{\prob}[1]{\rho_{#1}}
\newcommand{\rhotrain}{\rho_{\text{train}}^\alpha}

\newcommand{\tp}{\mathrm{TP}}
\newcommand{\fp}{\mathrm{FP}}
\newcommand{\ruc}{\mathrm{AUC}}

\newcommand{\dg}{D_{\mathrm{generated}}}
\newcommand{\dtrain}[1]{D_{\mathrm{train{#1}}}}

\newcommand{\kl}[2]{d_{\mathrm{KL}}\bra{#1,#2}}
\newcommand{\tv}[2]{d_{\mathrm{TV}}\bra{#1,#2}}
\newcommand{\df}[2]{d_{\Fc}\bra{#1, #2}}
\newcommand{\fvnorm}[2][\Fc]{\nm{#2}_{#1,1}}
\newcommand{\comcoe}[2]{\Gamma_{#1,#2}}
\newcommand{\comcoet}[3]{\Xi_{#1,#2,#3}}
\newcommand{\Span}[1]{span #1}
\NewDocumentCommand{\rademacher}{ O{\mu} O{m} m }{R_#2^{\bra{#1}}\bra{#3}}

\usepackage[nameinlink,capitalize]{cleveref}
\crefname{theorem}{Theorem}{Theorems}
\Crefname{theorem}{Theorem}{Theorems}
\crefname{proposition}{Proposition}{Propositions}
\Crefname{proposition}{Proposition}{Propositions}

\newtheorem{theorem}{Theorem}
\newtheorem{corollary}{Corollary}
\newtheorem{proposition}{Proposition}
\newtheorem{lemma}{Lemma}

\newcounter{packednmbr}

\title{On the Privacy Properties of GAN-generated Samples}

\usepackage{authblk}

\author{Zinan Lin\thanks{zinanl@andrew.cmu.edu}}
\author{Vyas Sekar\thanks{vsekar@andrew.cmu.edu}}
\author{Giulia Fanti\thanks{gfanti@andrew.cmu.edu \\ 
This work appeared in AISTATS 2021. This version has slightly modified the statement of Theorem \ref{thm:gan-pdp-nbr}.}}

\affil[]{Department of Electrical and Computer Engineering \\ Carnegie Mellon University}

\date{}

\begin{document}

\maketitle

\input{abstract}

\input{intro.tex}

\input{background.tex}

\input{dp.tex}

\input{proof.tex}

\input{membership}

\input{conclusion}

\section*{Acknowledgments}
This work was supported in part by National Science Foundation grants CIF-1705007 and CA-2040675. 
This material is based upon work supported by the Air Force Office of Scientific Research under award
number FA9550-21-1-0090.
This research was sponsored by the Combat Capabilities Development Command Army Research Laboratory and was accomplished under Cooperative Agreement Number W911NF-13-2-0045 (ARL Cyber Security CRA). The views and conclusions contained in this document are those of the authors and should not be interpreted as representing the official policies, either expressed or implied, of the Combat Capabilities Development Command Army Research Laboratory or the U.S. Government. The U.S. Government is authorized to reproduce and distribute reprints for Government purposes not withstanding any copyright notation here on.
The authors would also like to acknowledge the generous support of JP Morgan Chase, Siemens AG, Google, and the Sloan Foundation. 

\bibliographystyle{plainnat}
\bibliography{main}

\clearpage
\onecolumn

\section*{Appendix}
\appendix

\input{appendix.tex}

\end{document}

%% file: abstract.tex
\begin{abstract}
	The privacy implications of generative adversarial networks (GANs) are a topic of  great interest, leading to several recent algorithms for training GANs with privacy guarantees. 
	By drawing connections to the generalization properties of GANs, we prove that under  some assumptions,  GAN-generated samples inherently satisfy some (weak) privacy guarantees.
	First, we show that if a GAN is trained on $m$ samples and used to generate $n$ samples, the generated samples are  $(\epsilon,\delta)$-differentially-private for $(\epsilon,\delta)$ pairs where $\delta$ scales as $O(n/m)$.
	We show that under some special conditions, this upper bound is tight.  
	Next, we study the robustness of GAN-generated samples to membership inference attacks. 
	We model  membership inference  as a hypothesis test in which the adversary must determine whether a given sample was drawn from the training dataset or from the underlying data distribution. 
	We show that this adversary can achieve an area under the ROC curve that scales no better than $O(m^{-1/4})$.
\end{abstract}

%% file: intro.tex
\section{Introduction}\label{sec:intro}
Generative adversarial networks (GANs) are a class of generative models that aim to generate samples from a distribution $\mu$ given a database of training samples $D=\{x_i\}_{i=1}^m$, where the $x_i  \in X$ are drawn i.i.d. from a distribution $\mu$  \cite{goodfellow2014generative}. 
GANs are posed as a zero-sum game between two neural networks: a \emph{generator} that aims to generate new samples $\hat x_i \sim \mu$  and a \emph{discriminator} that aims to classify samples as being either real (i.e., a member of $D$) or generated.  
In recent years, GANs have garnered interest as a tool for generating synthetic data from potentially-sensitive raw datasets, such as medical data \cite{esteban2017real,jordon2018pate,choi2017generating}, banking transactions \cite{jordon2018pate}, and networking and server traces \cite{lin2020using}.
However, GANs are known to leak information about the data on which they were trained. Common concerns include memorization of sensitive training samples \cite{webster2019detecting} and membership inference attacks \cite{hayes2019logan,chen2019gan}.

In response to these concerns, a popular approach in the literature has been to train differentially-private GAN models. 
This can be accomplished through training methods like differentially-private stochastic gradient descent (DP-SGD) \cite{abadi2016deep}, PATE-GAN \cite{jordon2018pate}, and others \cite{long2019scalable,chen2020gs}. 
These approaches guarantee differential  privacy for releasing the \emph{model parameters}.  
In practice, they also degrade model quality due to the added noise, sometimes to the extent that the utility of data is completely destroyed \cite{lin2020using}. 

In some practical scenarios, an attacker may only have access to a GAN's \emph{generated samples}, instead of the \emph{model parameters} (e.g., when releasing synthetic datasets instead trained models \cite{choi2017generating,lin2020using}). 
Note that ensuring the privacy for releasing generated samples is a weaker condition than ensuring the privacy for releasing parameters
because of the post-processing property of differential privacy \cite{dwork2014algorithmic}. 
Releasing generated samples (versus model parameters) is therefore a promising solution for achieving better privacy-fidelity trade-offs.
As a first step towards this goal, it is important to understand the privacy properties of GAN-generated samples.

In this work, we show that GAN-generated samples satisfy an \emph{inherent} privacy guarantee without any special training mechanism, by making connections to recent results on the generalization properties of GANs \cite{zhang2017discrimination}. 
The strength of this privacy guarantee varies for different notions of privacy.
Specifically, we study two privacy notions: \emph{differential privacy (DP)}  and \emph{robustness to membership inference attacks}.

\begin{itemize}
	\item For differential privacy, we show that %
	the samples output by a vanilla GAN (i.e., trained without differential privacy) inherently satisfy a %
	(weak) differential privacy guarantee with high probability over the randomness of the training procedure. 
	More specifically, we consider the mechanism of training a GAN on $m$ training samples and using it to generate $n$ samples. We show that this mechanism ensures $\bra{\epsilon, \frac{O(n/m)}{\epsilon\bra{1-e^{-\epsilon}}}}$-differentially privacy %
	(under some assumptions). 
	The rate $O(n/m)$ shows that differential privacy guarantee is stronger as the training set size grows, but degrades as you generate more samples from the GAN.
	Additionally, we show that $O(1/m)$ rate is the optimal rate we can get when $n=1$.
	On the positive side, the results suggest that GAN-generated samples are inherently differentially private. On the negative side, however, the rate is as weak as releasing $n$ raw samples from $m$ training samples. This suggests that we need to incorporate other techniques (e.g., DP-SGD, PATE) in practice if we want meaningful differential privacy guarantees.
	
	\item For robustness to membership inference attacks, we show that vanilla GANs are \emph{inherently} robust to black-box membership inference attacks.
	We study the worst-case setting where an attacker can draw infinite samples from the trained GAN. Even in this case, we show that for any attacker strategy, the difference between the true positive rate and the false positive rate---as well as the area under the ROC curve (AUC)---both have upper bounds that scale as $O(m^{-\frac{1}{4}})$, where $m$ is the number of training samples. 
	This again means that GANs are more robust to membership inference attacks as training set size grows.
	More generally, for arbitrary generative models, we give a tight bound on the ROC region that relies on a simple, geometric proof.
	To the best of our knowledge, this is the first result to bound either the ROC region or the rate of decay of attack success probabilities for membership inference attacks, including for discriminative models.
\end{itemize}

The paper is organized as follows. We discuss the preliminaries and related work in \cref{sec:background}. Then we discuss the results on differential privacy and membership inference attacks in \cref{sec:dp} and \cref{sec:membership} respectively. Finally we conclude the paper in \cref{sec:conclusion}.

%% file: background.tex
\section{Background and Related Work}
\label{sec:background}

\paragraph{GANs} GANs are a type of generative model that implicitly learns a distribution from samples and synthesizes new, random samples \cite{goodfellow2014generative}. To achieve this goal, GANs use two components: a generator $g$ and a discriminator $f$. The generator maps a random vector $z$ from a pre-defined distribution $p_z$ (e.g., Gaussian or uniform) on a latent space to the sample space $X$; the discriminator reads a sample $x$, either from the real distribution $\mu$ or the generated distribution, and discriminates which distribution it comes from.
From another point of view, the discriminator is trying to estimate a distance between the generated distribution and the real distribution, and the generator tries to minimize the distance. 
More specifically, $g$ and $f$ are trained with the following loss:
\begin{align*}
	\min_g\max_f d_f(\mu, g_z),
\end{align*}
where $d_f$ is a distance measurement between distributions parametrized by $f$, and $g_z$ denotes the distribution obtained by sampling $z\sim p_z$ and mapping it through the generator $g(z)$. In the original  GAN paper \cite{goodfellow2014generative}, the authors use Jensen–Shannon divergence with formula
\begin{align*}
	d_f(\mu,g_z) =\Eb_{x\sim \mu} \brb{\log f(x)} + \Eb_{x\sim g_z} \brb{\log\bra{1-f(x)}}.
\end{align*}
Many other distances have later been used like the Kullback-Leibler divergence and Wasserstein distance \cite{nowozin2016f,gulrajani2017improved}.
After the invention of GANs, many approaches were proposed to improve their diversity and fidelity \cite{gulrajani2017improved,miyato2018spectral,lin2018pacgan,arjovsky2017towards}; today, GANs generate state-of-the-art realistic images  \cite{karras2017progressive,brock2018large}. 
Inspired by these early successes, there has been much interest in using GANs to share data in privacy-sensitive applications such as medical data \cite{esteban2017real,jordon2018pate,choi2017generating}, backing transactions \cite{jordon2018pate}, and networking and server traces \cite{lin2020using}.

\paragraph{Differential privacy}
Differential privacy (DP) has become the \emph{de facto} formal privacy definition in  many applications \cite{dwork2008differential,dwork2014algorithmic}. 
We say that two databases $D_0$ and $D_1$ are \emph{neighboring} if they differ in at most one element. 
A mechanism $M$ is $(\epsilon,\delta)$-differentially-private \cite{dwork2008differential,dwork2014algorithmic} if for any neighboring database $D_0$ and $D_1$, and any set $S\subseteq \range{M}$,
\begin{align*}
\Pb\brb{M(D_0)\in S} \leq e^\epsilon\Pb\brb{M(D_1)\in S} + \delta.
\end{align*}
Our analysis involves a stronger notion of differential privacy called \emph{probabilistic differential privacy}. 
A mechanism $M$ is $(\epsilon,\delta)$-probabilistically-differentially-private \cite{meiser2018approximate} if for any neighboring database $D_0$ and $D_1$, there exists sets $S_0\subseteq \range{M}$ where $\Pb\brb{M(D_0)\subseteq S_0} \leq \delta$, such that for any set $S\subseteq \range{M}$ 
\begin{align*}
\Pb\brb{M(D_0)\in S\setminus S_0} \leq e^\epsilon\Pb\brb{M(D_1)\in S\setminus S_0}.
\end{align*}
This says that $(\epsilon,0)$-differential-privacy condition holds except over a region of the support with probability mass at most $\delta$. 

It is straightforward to show that $(\epsilon,\delta)$-probabilistically-differential-privacy implies $(\epsilon,\delta)$-differential-privacy. 
In fact,     probabilistic differential privacy is \emph{strictly} stronger than differential privacy.
To see this, consider the following example: assume $\prob{p},\prob{q}$ are the distribution function of $M(D_0)$ and $M(D_1)$ respectively. When $e^\epsilon (1-\delta) < 1$ and $\delta > 0$, let $\epsilon' = \min\brc{1-e^\epsilon (1-\delta), \delta}$, we can construct the following $\prob{p}, \prob{q}$:
$\prob{q}(0)=0, \prob{p}(0)=1 - (1-\epsilon')e^{-\epsilon}$ and 
$\prob{q}(1)=1, \prob{p}(1)=(1-\epsilon')e^{-\epsilon}$.
Then $M$ satisfies $(\epsilon, \delta)$-differential-privacy, but does not satisfy $(\epsilon, \gamma)$-probabilistically-differential-privacy for any $\gamma < 1$.

\paragraph{Differential privacy and generalization}
The relation between differential privacy and generalization is well-studied  \cite{cummings2016adaptive, dwork2015preserving,bassily2016algorithmic,nissim2015generalization,wang2016learning}.
Three main factors differentiate our work:
  
(1)  Prior work has primarily studied how differential privacy implies various notions of generalization \cite{cummings2016adaptive, dwork2015preserving,bassily2016algorithmic,nissim2015generalization,wang2016learning}.
We  are interested in the other direction: when does generalization imply differential privacy?  
Cummings et al. showed that differential privacy is strictly weaker than perfect generalization but is strictly stronger than robust generalization \cite{cummings2016adaptive}. 
However, these results do not directly apply to generative models.  First, the results about perfect generalization only hold when the domain size is finite, which is less interesting for generative models. 
Second, the generalization definitions are based on hypotheses (i.e., conditional distributions), and are not easily extendable to generative models.

(2) More generally, prior work  has mainly considered discriminative models instead of generative models \cite{cummings2016adaptive,dwork2015preserving,bassily2016algorithmic,nissim2015generalization,wang2016learning}. \citet{wu2019generalization} showed that differential privacy implies generalization in GANs. 
However, the notion of generalization studied by \citet{wu2019generalization} is not meaningful for GANs. It captures the distance between the expected empirical loss of the discriminator and its actual loss, so we can have zero generalization error on the discriminator's loss (e.g., for a discriminator that always outputs 0 for any input) while having an arbitrary generated distribution. 
A more meaningful notion of generalization is to quantify the distance between the generated and real distributions. 
In this work, we use this latter notion of generalization.%

(3) The notion of differential privacy studied in prior work is with respect to releasing the \emph{parameters} of a model \cite{cummings2016adaptive, dwork2015preserving,bassily2016algorithmic,nissim2015generalization,wang2016learning,wu2019generalization}, whereas ours is with respect to releasing \emph{generated samples}.
This is a weaker guarantee than if one were to reason about the differential privacy guarantees of a generative model's parameters, but is relevant to many application settings.

\paragraph{Membership inference attacks} 
Membership inference attacks are closely related to differential privacy. 
Given a trained model, a membership inference attack aims to infer whether a given sample $x$ was in the training dataset or not.
The main difference between membership inference and differential privacy is that the attacker in differential privacy is assumed to know an adversarially-chosen pair of candidate training databases,
whereas in membership inference attacks, the adversary is typically given access only to test samples (of which some are training samples) and the model. 
In some cases, the attacker is also given side information about the number of training samples in the test set.
Hence, in general, the attacker in membership inference is neither strictly weaker nor strictly stronger than the differential privacy attacker. 

There have been many membership inference attacks proposed for discriminative models \cite{sablayrolles2019white,li2020label,long2018understanding,melis2019exploiting,salem2018ml,shokri2017membership,yeom2018privacy} and generative models (including GANs) \cite{hayes2019logan,chen2019gan,hilprecht2019monte}. Therefore, understanding robustness of GANs to membership inference attacks is important. 

There has been some theoretical analysis on membership inference attacks \cite{sablayrolles2019white,farokhi2020modelling}. 
For example, \citet{sablayrolles2019white} show a theoretically optimal strategy for membership inference attacks. 
\citet{farokhi2020modelling} show that the accuracy of membership inference attacks for a particular sample is upper bounded by the Kullback–Leibler divergence between the distributions of parameters with and without that sample. 
However, these results do not give a practical method for bounding an attacker's global performance (i.e., the ROC curve).
In this paper, we resolve the issue. Ours is the first work to show that the success rate of membership inference attack decays as the number of training samples $m$ grows, whereas prior work \cite{farokhi2020modelling} only shows that the limit of the success rate is 0.5 as  $m\to \infty$. 

%% file: dp.tex
\section{Bounds on Differential Privacy}
\label{sec:dp}

 We start with some notation, definitions, and assumptions. 
For a probability measure $\mu$ on $X$, we let $\prob{\mu}$ denote its density function.
We use $\Gc$ and $\Fc$ to denote the set of possible generators and discriminators, respectively, where $\Fc$ is a set of functions $X\to \Rb$.
Our results rely on three assumptions:
\begin{enumerate}
	\item[(A1)] Our generator set $\Gc$ and  discriminator set $\Fc$  satisfy $\forall \nu_1,\nu_2\in \Gc, \log\bra{\nicefrac{\prob{\nu_1}}{\prob{\nu_2}}} \in \Span{\Fc}$, where $\Span{\Fc}$ is defined as the set of linear combinations of the functions:
	\begin{align*}
	\Span{\Fc} \triangleq \brc{w_0+\sum_{i=1}^{n}w_if_i: w_i\in \Rb, f_i\in \Fc, n\in \Nb}.
	\end{align*}
	\item[(A2)] Our discriminator set $\Fc$ is even, i.e, $\forall f\in\Fc$, $-f\in\Fc$, and assume that $\forall f \in \Fc,~\forall x\in X$, 
	$$\infnm{f(x)}\leq \Delta,$$ 
	where $\infnm{f} \triangleq \sup_{x\in X} |f(x)|$.
	\item[(A3)] The discriminator set $\Fc=\brc{f_\theta: \theta\in \Theta\subseteq [-1,1]^p}$ and 
	$$\infnm{f_\theta-f_{\theta'}}\leq L\snm{\theta-\theta'}.$$

\end{enumerate}
According to Lemma 4.1 in \citet{bai2018approximability}, when generators in $\Gc$ are invertible neural networks (e.g., in \citet{dinh2016density,behrmann2019invertible}) with $l$ layers, then the discriminator set $\Fc$ of neural networks with $l+2$ layers satisfy (A1). 
Assumption (A2) is easily satisfied by neural networks with an activation function on the output layer that bounds the output (e.g., sigmoid).
(A3) assumes the discriminator is Lipschitz in its (bounded) parameters; 
several recent works attempt to make network layers Lipschitz in inputs through various forms of regularization (e.g., spectral normalization \cite{miyato2018spectral,lin2020spectral}), which makes the network Lipschitz also in parameters \cite{lin2020spectral}.

 Given $\mu,\nu$, two probability measures on $X$, and set $\Fc$ of functions $X\to \Rb$, the  \emph{integral probability metric} \cite{muller1997integral} is defined as:
\begin{align*}
\df{\mu}{\nu} \triangleq \sup_{f\in \Fc}\brc{\Eb_{x\sim \mu}\brb{f(x)} - \Eb_{x\sim \nu}\brb{f(x)}}.
\end{align*}
Given function set $\Fc$, the \emph{$\Fc$-variation norm} of function $g$ is defined as
\begin{align*}
\fvnorm{g} \triangleq \inf\bigg\{\sum_{i=1}^{n}\abs{w_i}: g=w_0 + \sum_{i=1}^{n}w_if_i,
\forall n\in \Nb, w_i\in\Rb, f_i\in\Fc\bigg\},
\end{align*}
which intuitively describes the complexity of linearly representing $g$ using functions in $\Fc$ \cite{zhang2017discrimination}.
We define
\begin{align}
\comcoe{\Fc}{\Gc} \triangleq \sup_{\nu_1,\nu_2 \in \Gc}\fvnorm{\log\bra{\nicefrac{\rho_{\nu_1}}{\rho_{\nu_2}}}} ,\label{eq:Gamma}
\end{align}
which intuitively bounds the complexity of representing differences in log densities of pairs of generators in $\Gc$ using functions in $\Fc$.

We consider the GAN training and sampling mechanism in \cref{alg:dp-gan}, which adds a sampling processing before the normal GAN training.
Note that the sampling process in \cref{alg:dp-gan} is commonly used in existing GAN implementations \cite{goodfellow2014generative}, which typically sample i.i.d. from $D$ in each training batch. 
We have moved this sampling process to the beginning of training in \cref{alg:dp-gan} for ease of analysis.

\IncMargin{1.5em}
\begin{algorithm}[ht]
	\SetKwInOut{Input}{Input}
	\SetKwInOut{Output}{Output}
	
	\Indm
	\Input{$D$: A training dataset containing $m$ samples. \newline 
		       $k$: Number of sampled training samples used in training. \newline
	           $n$: Number of generated samples.}
	\Output{$\dg$: $n$ generated samples.}
	\BlankLine
	
	\Indp
	$\dtrain{} \leftarrow$ $k$ i.i.d. samples from $D$ \label{alg:dp-gan-sample}\;
	$g \leftarrow$ Trained GAN using $\dtrain{}$\;
	$\dg{} \leftarrow$ $n$ generated samples from the trained $g$\;
	\caption{Differentially-private GAN mechanism.}
	\label{alg:dp-gan}
\end{algorithm}
\DecMargin{1.5em}

Finally, we define the quantity
\begin{align}
	\ferrlong{k,\xi,\mu} \triangleq 2\bra{\inf_{\nu\in\Gc}\df{\mu}{\nu} + \opterr + \frac{C_\xi}{\sqrt{k}}},
	\label{eq:epsilon-f}
	\end{align}
	where 
	\begin{align}
	C_\xi=16\sqrt{2\pi}pL + 2\Delta \sqrt{2\log(1/\xi)},
	\label{eq:c_xi}
	\end{align} 
	$p$ and $L$ are the constants defined in (A3),  $\Delta$ is the constant defined in (A2),  $\opterr$ is an upper bound on the optimization error, i.e.,
	$$\df{\hat{\mu}_k}{g}-\inf_{\nu\in\Gc}\df{\hat{\mu}_k}{\nu} \leq \opterr, $$
$\mu$ is the real distribution,  $\hat{\mu}_k$ be the empirical distribution of $\mu$ on $k$ i.i.d training samples,  $g$ is the trained generator from the optimization algorithm.
The interpretation of $\ferrlong{k,\xi,\mu}$ is described in greater detail in \cref{sec:prf-gan-pdp-nbr}, 
but intuitively, it will be used to bound %
a GAN's generalization error arising from  approximation, optimization, and sampling of the training datasets in \cref{alg:dp-gan-sample}. 
To reduce notation, we will henceforth write this quantity as $\ferr{k,\xi}$.

Our main result states that a GAN trained on $m$ samples and used to generate $n$ samples satisfies a $\bra{\epsilon, \frac{O(n/m)}{\epsilon\bra{1-e^{-\epsilon}}}}$-differential-privacy guarantee; 
moreover, this bound is tight when $n=1$ and for small $\epsilon$. 
\begin{theorem}[Achievability]\label{thm:gan-pdp-nbr}
 Consider a GAN trained on $m$ i.i.d. samples from distribution $\mu$. 
The mechanism in \cref{alg:dp-gan} under assumptions (A1)-(A3) satisfies $\bra{\epsilon,\delta}$-differential privacy for any $\epsilon>0$, $\xi >0$, and 
\begin{align}
\delta~  ~>  \frac{n  ~  \comcoe{\Fc}{\Gc}   }{\epsilon(1-e^{-\epsilon})}~\bra{\frac{2\Delta}{m} + \ferr{k,\xi}} +2\xi ~.
\label{eq:s}
\end{align}
\end{theorem}
We assume that the terms regarding $\xi$ ($2\xi$ and $\ferr{k,\xi}$) can be made negligible: for any $\xi>0$,
$\frac{C_\xi}{\sqrt{k}}$ can be arbitrarily small as we get more samples from the sampling phase in \cref{alg:dp-gan-sample}, 
and we assume negligible approximation error and optimization error.
Hence, the dominating term in \cref{eq:s}
 scales as $O (n/m)$; here we are ignoring the dependency on $\epsilon$.
Next, we show that for a special case where $n=1$ and $\epsilon$ scales as $O\bra{\frac{1}{m}}$, this bound is tight in an order sense (again ignoring dependencies on $\epsilon$).

\begin{proposition}[Converse for $n=1$]\label{thm:gan-pdp-nbr-upper}
	Under the assumptions of \cref{thm:gan-pdp-nbr}, 
	let 
	$$\Delta' = \sup_{f\in \Fc} \sup_{x,y\in X} |f(x) -f(y)|.$$ 
	Then with probability at least $1-2\xi$, the GAN mechanism in \cref{alg:dp-gan} for generating 1 sample (i.e., $n=1$) does not satisfy $(\epsilon, \delta)$-differential-privacy for any 
	\begin{align}
	\delta <\frac{\bra{e^\epsilon +1}}{2\Delta }\bra{\frac{\Delta'}{2m} - \ferr{k,\xi} } + 1 - e^\epsilon.
	\label{eq:converse}
	\end{align}
\end{proposition}
This bound in \cref{eq:converse} is non-vacuous (nonnegative) when
$\epsilon < \frac{1}{\Delta} \bra{\frac{\Delta'}{2 m} -\ferr{k,\xi}}$
and $m< \nicefrac{\Delta'}{2\ferr{k,\xi}}$.
Again assuming $\ferr{k,\xi}\approx 0$ (i.e., ignoring the approximation error and optimization error and taking  $\frac{C_\xi}{\sqrt{k}} \to 0$), 
the latter condition holds trivially.
Hence when $\epsilon$ scales as $O\bra{\frac{1}{m}}$,
it is not possible to achieve an $(\epsilon, \delta)$-probabilistic differential privacy guarantee for $\delta = o\bra{\frac{1}{m}}$.

\paragraph{Discussion}
These results suggest that GAN-generated samples satisfy an inherent differential privacy guarantee, so the influence of any single training sample on the final generated samples is bounded.
However, the rate $O(n/m)$ is  weak. 
For comparison, the mechanism that  releases $n$ samples uniformly at random from a set of $m$ training samples satisfies $(0, n/m)$-differential privacy, which is of the same rate.
Therefore, $(\epsilon,\delta)$-differential privacy usually requires $\delta \ll \frac{1}{\text{poly}(m)}$ to be meaningful.
To satisfy this condition, we would need $\epsilon$ to grow as a function of $m$ (since $\delta$ in our results is a function of $\epsilon$), which is not practically viable. 
This suggests the need for incorporating additional techniques (e.g., DP-SGD, PATE)  to achieve stronger differential privacy guarantees in practice.

Note that in the typical GAN training process, we usually have $k>m$ (i.e., the number of sampled training samples is  larger than the dataset size). Therefore,  the random sampling step in \cref{alg:dp-gan-sample} does not give privacy amplification due to subsampling (e.g., \cite{balle2018privacy}). 
However, if the size of the training dataset is large enough that we do not need to sample all dataset entries to achieve good generalization, we can apply sub-sampling theorems \cite{balle2018privacy} to tighten the bounds in \cref{thm:gan-pdp-nbr}.

%% file: proof.tex
\subsection{Proof of \cref{thm:gan-pdp-nbr}}
\label{sec:prf-gan-pdp-nbr}
Assume that the two neighboring datasets are $D_0$ and $D_1$, whose empirical distributions are $\hat{\mu}_m^0$ and $\hat{\mu}_m^1$, and the trained generator distributions from \cref{alg:dp-gan} are $g_0$ and $g_1$ respectively. 
The proof has two parts. 
First, we upper bound the distance between $g_0$ and $g_1$ by building on prior generalization results. 
Then, we use this distance to prove a %
 differential privacy guarantee.

\begin{lemma}[Error of neural-network discriminators, Corollary 3.3 in \citet{zhang2017discrimination}] \label{thm:error_neural}
	Let $\mu$ be the real distribution, and $\hat{\mu}_k$ be the empirical distribution of $\mu$ on $k$ i.i.d training samples. Define the trained generator from the optimization algorithm as $g$ and assume the optimization error is bounded by $\opterr$, i.e., $\df{\hat{\mu}_k}{g}-\inf_{\nu\in\Gc}\df{\hat{\mu}_k}{\nu} \leq \opterr$.
	Under assumptions (A2) and (A3),
	then with probability at least $1-\xi$ w.r.t. the randomness of training samples, we have
	\begin{align*}
	\df{\mu}{g} ~\leq~  \frac{1}{2}\ferr{k,\xi} ~=
	\underbrace{\inf_{\nu\in\Gc}\df{\mu}{\nu}}_{\textup{approximation error}} + 
	\underbrace{\opterr}_{\textup{optimization error}} +
	\underbrace{\frac{C_\xi}{\sqrt{k}}}_{\textup{generalization error}},
	\end{align*}
	where $C_\xi$ is defined in \cref{eq:c_xi}.
\end{lemma}

From this lemma, we know that with high probability, $\df{\hat{\mu}_m^i}{g_i}$ is small. The next lemma states that $\df{\hat{\mu}_m^0}{\hat{\mu}_m^1}$ is also small.

\begin{lemma}\label{thm:error_neural_between_g_neighboring}
	Under the assumption (A2),
    for any two neighboring datasets $D_0$, $D_1$ with $m$ samples, we have 
	$$\df{\hat{\mu}_m^0}{\hat{\mu}_m^1} \leq \frac{2\Delta}{m} \;,$$.
\end{lemma}
(Proof in \cref{app:thm:error_neural_between_g_neighboring}.)
Next, we use these results to argue that $\df{g_0}{g_1}$ is small with high probability. 

\begin{lemma} \label{thm:error_neural_between_g}
	Assume we have two training sets $D_0$ and $D_1$, and the trained generator distributions using $D_0$ and $D_1$ with \cref{alg:dp-gan} are $g_0$ and $g_1$, respectively. Under the assumption of \cref{thm:error_neural} and \cref{thm:error_neural_between_g_neighboring}, we have that with probability at least $1-2\xi$, 
	\begin{align*}
	\df{g_0}{g_1} \leq \ferr{k,\xi} + \frac{2\Delta}{m}.
	\end{align*}
\end{lemma}
(Proof in \cref{app:thm:error_neural_between_g}.)
We next use this bound on the integral probability metric to bound Kullback-Leibler (KL) divergence between $g_0$ and $g_1$ with the following lemma.

\begin{lemma} \label{thm:kl_base}
Given a generator set $\Gc$ and a discriminator set $\Fc$ which satisfy assumption (A1), then we have $\forall \nu_1,\nu_2\in \Gc$
\begin{align*}
    \kl{\nu_1}{\nu_2} + \kl{\nu_2}{\nu_1} \leq \comcoe{\Fc}{\Gc} \df{\nu_1}{\nu_2} 
\end{align*}
where
$
    \comcoe{\Fc}{\Gc}
$ is defined in \cref{eq:Gamma}
and $\kl{\cdot}{\cdot}$ is the Kullback–Leibler divergence.
\end{lemma}
	This follows directly from Proposition 2.9 in \citet{zhang2017discrimination}, which states the following.
Denote $\mu$'s and $\nu$'s density functions as $\prob{\mu}$ and $\prob{\nu}$ respectively. If $\log\bra{\nicefrac{\prob{\mu}}{\prob{\nu}}} \in \Span{\Fc}$, then we have
\begin{align*}
\kl{\mu}{\nu} + \kl{\nu}{\mu} \leq \nm{\log\bra{\nicefrac{\prob{\mu}}{\prob{\nu}}}}_{\Fc,1} \df{\mu}{\nu}.
\end{align*}
Note that $\comcoe{\Fc}{\Gc} =1$ when generators in $\Gc$ are invertible neural networks with $l$ layers and discriminator set $\Fc$ is $(l+2)$-layer neural networks, according to Lemma 4.1 in \citet{bai2018approximability}.

    Following \cref{thm:kl_base} and \cref{thm:error_neural_between_g}, immediately we have that with probability at least $1-2\xi$,
    \begin{align*}
        \kl{g_0}{g_1} + \kl{g_1}{g_0} \leq
        \comcoe{\Fc}{\Gc}  \bra{\ferr{k,\xi} +\frac{2\Delta}{m}}
    \end{align*}
    and 
    \begin{align*}
         \kl{g_0^n}{g_1^n} + \kl{g_1^n}{g_0^n} \leq
        n\cdot\comcoe{\Fc}{\Gc}  \bra{\ferr{k,\xi} +\frac{2\Delta}{m}}
    \end{align*}

Then, we connect KL divergence with differential privacy:

\begin{lemma}\label{thm:kl-probdp}
    If a mechanism $M$ satisfies that for any two neighboring databases $D_0$ and $D_1$, $\kl{p}{q}+\kl{q}{p}\leq s$, where $p,q$ are the probability measure of $M(D_0)$ and $M(D_1)$ respectively, then M satisfies $(\epsilon, \frac{s}{\epsilon(1-e^{-\epsilon})})$-probabilistic-differential-privacy for all $\epsilon>0$. 
\end{lemma}
(Proof in \cref{app:thm:kl-probdp}.)

Note that probabilistic differential privacy implies differential privacy (\cref{sec:background}). From the above lemmas, we know that the mechanism in \cref{alg:dp-gan} under assumptions (A1)-(A3) satisfies $\bra{\epsilon,\delta}$-differential-privacy for any $\epsilon>0$ and 
$
	\delta \geq \frac{n  ~  \comcoe{\Fc}{\Gc}   }{\epsilon(1-e^{-\epsilon})}~\bra{\frac{2\Delta}{m} + \ferr{k,\xi}}
$,
with probability at least $1-2\xi$ over the randomness in \cref{alg:dp-gan-sample}.
The final step is to move the low probability of failure ($2\xi$) into the $\delta$ term.

\subsection{Proof of \cref{thm:gan-pdp-nbr-upper}}

The proof has two parts. 
First, we lower bound the distance between $g_0$ and $g_1$ by building on prior generalization results. 
Then, we use this distance to prove the the lower bound of $\delta$ for $(\epsilon,\delta)$-differential-privacy.

Because $\Delta' = \sup_{f\in \Fc} \sup_{x,y\in X} |f(x) -f(y)|$, we know that there exists $f'\in \Fc$ and $x_1,y_1\in X$ such that $|f'(x_1)-f'(y_1)|\geq \frac{\Delta'}{2}$. Let's construct two databases: $D_0=\brc{x_1,...,x_{m}}$ and $D_1=\brc{y_1,x_2,...,x_{m}}$ where $x_2,...,x_{m}$ are arbitrary samples from $X$. Then we have
\begin{align*}
	\df{\hat{\mu}_m^0}{\hat{\mu}_m^1} = \frac{1}{m} \sup_{f\in \Fc} \brc{ f(x_1) - f(y_1) } \geq \frac{\Delta'}{2m}
\end{align*}
where the proof of the first equality is in \cref{app:thm:error_neural_between_g_neighboring}.

On the other hand, from \cref{thm:error_neural}, we know that with probability at least $1-\xi$, $\df{\hat{\mu}_m^i}{g_i} \leq \frac{1}{2}\ferr{k,\xi}$ for $i=0,1$. 
Combining the above, we have the following lemma.

\begin{lemma} \label{thm:error_neural_between_g_upper}
	Assume we have two training sets $D_0$ and $D_1$, and the trained generator using $D_0$ and $D_1$ with \cref{alg:dp-gan} are $g_0$ and $g_1$ respectively. Under the assumption of \cref{thm:error_neural} and \cref{thm:error_neural_between_g_neighboring}, we have that with probability at least $1-2\xi$, 
	\begin{align*}
	\df{g_0}{g_1} \geq \frac{\Delta'}{2m} - \ferr{k,\xi}
	\end{align*}
\end{lemma}
(Proof in \cref{app:thm:error_neural_between_g_upper}.)
Note that this lower bound is nonnegative only for $m< \nicefrac{\Delta'}{2\ferr{k,\xi}} $.

Now we connect integral probability metric to total variation (TV) distance. Because the discriminators are bounded (A2), we have 
\begin{align*}
	\df{g_0}{g_1}\leq 2\Delta \tv{g_0}{g_1}
\end{align*}
where $\tv{g_0}{g_1}$ is the TV distance between $g_0$ and $g_1$.
From the above, we know that 
\begin{align*}
\tv{g_0}{g_1}\geq \frac{\Delta'}{4m\Delta } - \frac{\ferr{k,\xi}}{2\Delta}
\end{align*}

Finally, we connect TV distance to differential privacy with the following lemma.

\begin{lemma}\label{thm:dp_tv}
	If a mechanism satisfies $(\epsilon, \delta)$-differential-privacy, then for any two neighboring databases $D_0$ and $D_1$, we have $$\tv{p}{q}\leq \frac{e^\epsilon + 2\delta-1}{e^\epsilon+1}$$ where $p,q$ are the probability measure of $M(D_0)$ and $M(D_1)$ respectively.
\end{lemma}

Therefore, we have
\begin{align*}
	\delta \geq \frac{\bra{e^\epsilon +1}\Delta'}{4m\Delta } - \frac{\bra{e^\epsilon + 1}\ferr{k,\xi} }{2\Delta} + 1 - e^\epsilon.
\end{align*}

%% file: membership.tex
\section{Bounds on Robustness to Membership Inference Attacks}
\label{sec:membership} 
In this section, we first derive a general bound on the error of membership inference attacks for generative models. Then, we utilize generalization bounds for GANs to obtain specific bounds for GANs.

We focus on the black-box attack setting \cite{sablayrolles2019white,chen2019gan,hayes2019logan}, in which the attacker can sample from the generated distribution $g_\alpha$, but does not have access to the generator parameters $\alpha$. To upper bound the attack performance, we assume that attacker has unlimited resources and can access the trained generator infinite times, so that it can accurately get the generated distribution $g_\alpha$.

Our analysis departs from prior analysis of membership inference in discriminative models \cite{sablayrolles2019white} in two key respects:
\begin{itemize}
	\item \cite{sablayrolles2019white} assume that the attacker has access to a dataset $U=\{u_1, \ldots, u_p\}$, which contains all the training samples (i.e., $D \subseteq U$) and some other test samples drawn from the ground-truth distribution $\mu$ (so $p > m$). 
	It also assumes that the attacker knows the number of training samples $m$. We argue that this assumption is too strong, especially in the case of generative models, where training samples are typically proprietary. Therefore, we assume that the attacker makes guesses purely based on a single test sample $x \in X$, without access to such a dataset. The test sample is either drawn from the ground-truth distribution (i.e., $x\sim \mu$), or from the training dataset (i.e., $x \stackrel{\mathrm{i.i.d.}}{\longleftarrow} D$).
	\item The analysis in \cite{sablayrolles2019white} focuses on the quantity $\Pb\bra{u\in D |\alpha}$ for a particular $u$. This is useful for finding an attack policy, but is not conducive to characterizing the error statistics. 
	Instead, we want to be able to bound the shape of ROC curve. That is, we want to upper bound the true positive rate an attacker can achieve given any desired false positive rate. 
	We show that this problem can be reduced to a clean hypothesis testing problem, whose errors are closely tied to the generalization errors of GANs.
\end{itemize}

Following prior work on the theoretical analysis of membership inference attacks \cite{sablayrolles2019white}, we assume that the distribution of the generator parameters is 
\begin{align}
	\Pb\bra{\alpha | x_1,...,x_m} \propto e^{-\sum_{i=1}^{m}\ell\bra{\alpha, x_i}}
	\label{eq:mi_generator_dist}
\end{align}
 (setting $T=1$ in \citet{sablayrolles2019white}), where $x_1,...,x_m\sim \mu$ are i.i.d training samples drawn from the ground-truth distribution $\mu$, and $\ell(\alpha,x)$ denotes the loss on sample $x$ and parameter $\alpha$.  As we are focusing on generative models here, we assume that the loss is Kullback-Leibler (KL) divergence, i.e., $\ell\bra{\alpha,x} = \log \bra{\nicefrac{\prob{\mu}\bra{x_i}}{\prob{g_\alpha}\bra{x_i}}}$, 
where $\prob{g_\alpha}$ denotes the density of the generator with parameters $\alpha$. 
Note that many generative models are explicitly or implicitly minimizing this KL divergence, including some variants of GANs (more specifically, f-GANs with a specific loss \cite{nowozin2016f}), Variational Autoencoder (VAE) \cite{kingma2013auto},  PixelCNN/PixelRNN \cite{oord2016pixel}, and many other methods that are based on maximum likelihood \cite{deep-learning}. We use KL divergence also because this simplifies the analysis and highlights key theoretical insights. With this assumption, the parameter distribution becomes
\begin{align*}
	\Pb\bra{\alpha | x_1,...,x_m} \propto \prod_{i=1}^{m} \frac{\prob{g_\alpha}\bra{x_i}}{\prob{\mu}\bra{x_i}}\;.
\end{align*}

Let $\rhotrain{}$ denotes the density posterior distribution of the training samples given parameter $\alpha$. The following proposition shows that this distribution takes a simple form.

\begin{proposition}[Posterior distribution of training samples]
	The posterior distribution of training samples is equal to the generated distribution, i.e.,
	$
		\rhotrain{} = \prob{g_\alpha}
	$.
\end{proposition}
\begin{proof}
	For any $x$, we have
	
		\begin{align*}
			\quad\rhotrain{}\bra{x} &= \frac{\Pb\bra{x \in D, \text{parameter is }\alpha}}{\Pb\bra{\text{parameter is }\alpha}}\\
			&= \frac{\prob{\mu}\bra{x} \int_{x_2,...,x_m} \prod_{i=2}^{m}\prob{\mu}\bra{x_i} \Pb\bra{\alpha|x,x_2,...,x_m} dx_2\ldots dx_m}{\int_{x_1,...,x_m} \prod_{i=1}^m \prob{\mu}(x_i) \Pb\bra{\alpha|x_1,...,x_m} dx_1\ldots dx_m}\\
			&=\prob{g_\alpha}\bra{x}
		\end{align*}
	
\end{proof}

This proposition validates prior membership inference attacks that utilize approximations of $\prob{g_\alpha}\bra{x}$ to make decisions \cite{chen2019gan,hayes2019logan,hilprecht2019monte}.

With this proposition, the problem becomes clear: for a given sample $x$, the attacker needs to decide whether the sample comes from the training set (i.e., from $g_\alpha$) or not (i.e., from $\mu$). 
In the following theorem, we outer bound the ROC region for this hypothesis test, which relates the true positive (TP) rate to the false positive (FP) rate.

\begin{proposition} \label{thm:roc}
	 Consider a generative model $g_\alpha$ and a real distribution $\mu$. Define $r \triangleq d_{TV}(g_\alpha,\mu)$ as the total variation (TV) distance between the two distributions.  
	 Define function $f: [0,1]\to[0,1]$ as
	\begin{align*}
	f(x) = \left\{ \begin{matrix}
	x+r&(0\leq x \leq 1-r)\\
	1&(r<x\leq 1)
	\end{matrix}\right.,
	\end{align*}
	Then we have that for any membership inference attack policy $\mathcal A$, the ROC curve $g_{\mathcal A}:[0,1]\to [0,1]$ (mapping FP to TP) satisfies $g(x)\leq f(x),\forall~ 0\leq x\leq 1$, i.e., the ROC curve is upper bounded by $f$. Also, this bound is tight, i.e., there exists two distributions $\mu',g'$ such that $\tv{\mu'}{g'}=r$ and the ROC curve is exactly $f$ at every point.
\end{proposition}

(Proof in \cref{proof:roc})
As a result, we can directly bound the area under the ROC curve (AUC) as a function of the total variation distance.

\begin{corollary}[Bound on the AUC for generative models] \label{thm:auc-tv}
For any attack policy on a generative model, we have 
	\begin{align*}
		\ruc{} \leq -\frac{1}{2} \tv{g_\alpha}{\mu}^2 + \tv{g_\alpha}{\mu} + \frac{1}{2}.
	\end{align*}
\end{corollary}

Note that \cref{thm:roc} and \cref{thm:auc-tv} hold for any generative model. For GANs in particular, we can use generalization bounds in \cref{thm:error_neural} to obtain the following result. 

\begin{theorem} \label{thm:roc-gan}
	Consider a GAN model $g_\alpha$ and a real distribution $\mu$. Define $$
	\comcoet{\Fc}{\Gc}{\mu} \triangleq \sup_{\nu \in \Gc}\fvnorm{\log\bra{\nicefrac{\prob{\mu}}{\prob{\nu}}}}
	$$
	and 
	\begin{align}
	\tverr{m,\delta} \triangleq \frac{\sqrt{ \comcoet{\Fc}{\Gc}{\mu} \cdot \ferr{m,\delta}}}{2\sqrt{2}}\;,
	\label{eq:epsilon-tv}
	\end{align}  
	where $\ferr{m,\delta}$ is defined as in \cref{eq:epsilon-f}.
	Define function $f: [0,1]\to[0,1]$ as
	\begin{align*}
	f(x) = \left\{ \begin{matrix}
	x+\tverr{m,\delta}&(0\leq x \leq 1-\tverr{m,\delta})\\
	1&(\tverr{m,\delta}<x\leq 1)
	\end{matrix}\right. .
	\end{align*}
	Then we have that for any membership inference attack policy $\mathcal A$, the ROC curve $g_{\mathcal A}:[0,1]\to [0,1]$  satisfies $g(x)\leq f(x),\forall~ 0\leq x\leq 1$, and the bound is tight. 
\end{theorem}
One complication is that existing generalization bounds do not directly bound TV distance, so these must be extended. The proof can be found in \cref{proof:roc-gan}, and directly gives the following corollary bounding the AUC for GANs.

\begin{corollary}[Bound on AUC for GANs] For any attack policy on GANs, we have that with probability at least $1-\delta$ w.r.t. the randomness of training samples,
	\begin{align*}
	\ruc{} \leq -\frac{1}{2} \tverr{m,\delta}^2 + \tverr{m,\delta} + \frac{1}{2},
	\end{align*}
	where $\tverr{m,\delta}$ is defined as in \cref{eq:epsilon-tv}.
\end{corollary}

Note that the AUC bound decays as $O(m^{-1/4})$. 

\paragraph{Discussions} These results confirm the prior empirical observation that GANs are more robust to membership inference attacks when the number of training samples grows \cite{lin2020using,chen2019gan}.
However, the results heavily rely on the assumption of generator parameter distribution \cref{eq:mi_generator_dist}, which was introduced in  \cite{sablayrolles2019white}. It is unlikely to strictly hold in practice. Extending the results to more general settings would be an interesting future direction.

\subsection{Proof of \cref{thm:roc}}
\label{proof:roc}
	It is known from the hypothesis testing literature \cite{kairouz2015composition} that, for any attack policy, the difference between the true positive rate (TP) and false positive rate (FP) is upper bounded by the total variation (TV) distance $\tv{g_\alpha}{\mu}$:
	\begin{align}
	\tp{} \leq \fp{} + \min\brc{\tv{g_\alpha}{\mu}, 1-\fp{}}\;.
	\label{eq:tp-bound}
	\end{align}

	Note that total variation distance and ROC curve has a very simple geometric relationship, as noted in \citet{lin2018pacgan} (Remark 7).
		That is, the total variation distance between $g_\alpha$ and $\mu$ is the intersection between the vertical axis and the tangent line to the upper boundary of the ROC curve that has slope 1 (e.g., see \cref{fig:roc}).
	This immediately implies that $f(x)$ is an upper bound for all possible ROC curves.
	\begin{figure}[t]
		\centering
		\includegraphics[width=0.3\linewidth]{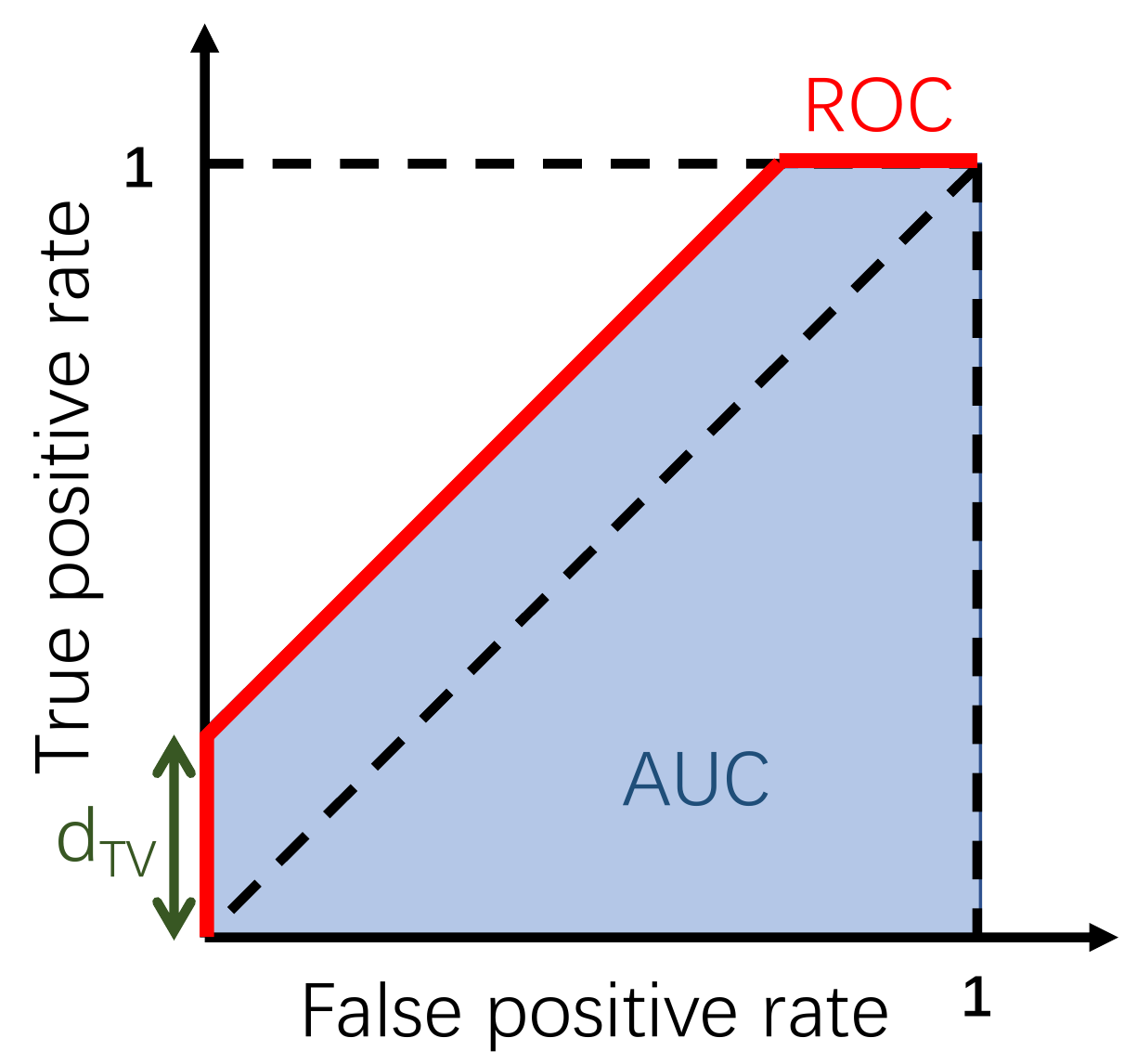}
		\caption{The upper bound of ROC curves.}
		\label{fig:roc}
	\end{figure}
	
	To show tightness, we can construct a $g'$ and $\mu'$ as shown in \cref{fig:dist}, such that $\tv{\mu'}{g'}=r$ and they achieve the ROC curve in \cref{fig:roc}.
	
	\begin{figure}[t]
		\centering
		\includegraphics[width=0.32\linewidth]{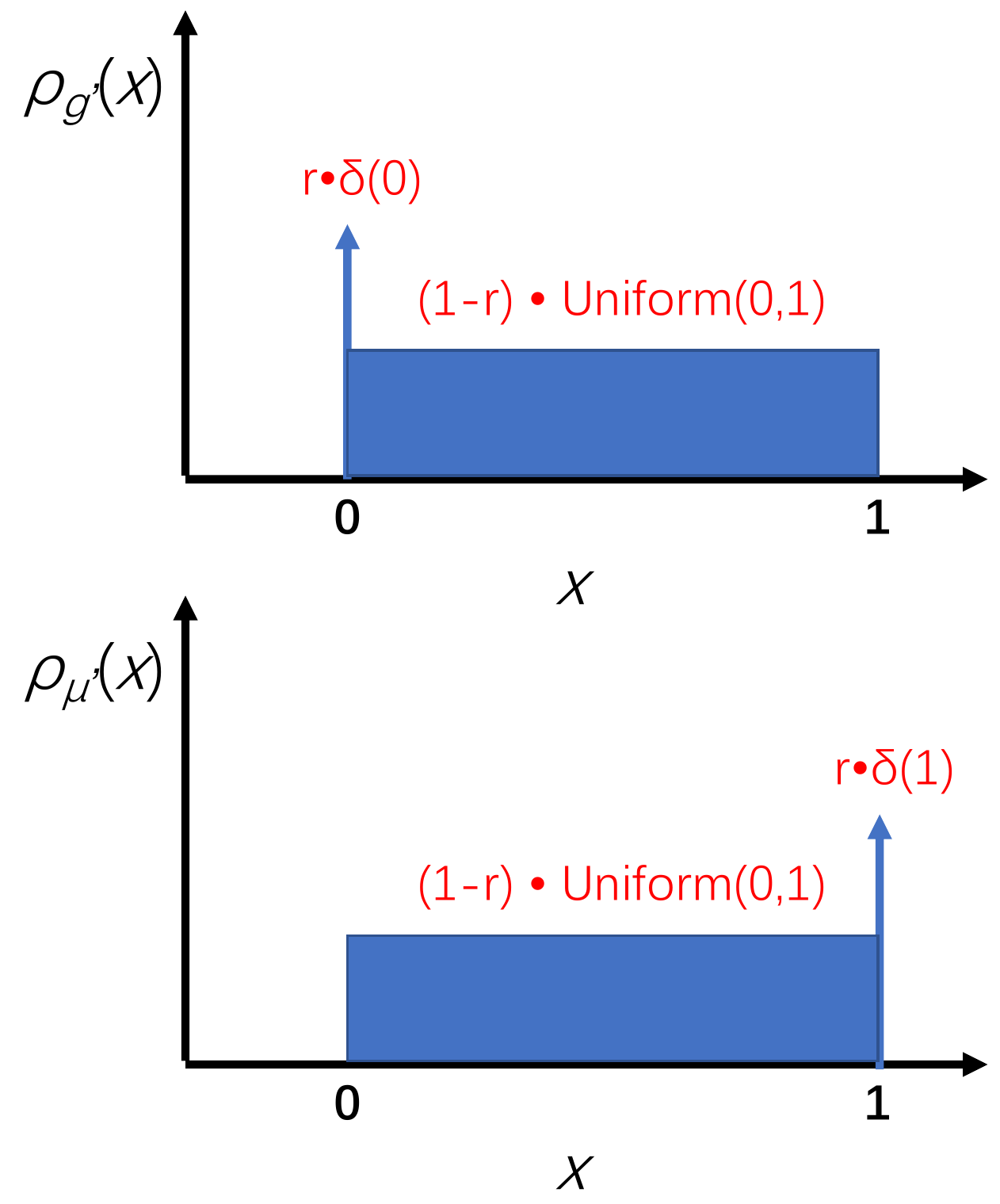}
		\caption{The pair of distributions that achieve the ROC upper bound.}
		\label{fig:dist}
	\end{figure}

\subsection{Proof of \cref{thm:roc-gan}}
\label{proof:roc-gan}
	We begin by showing that, under the assumptions in \cref{thm:error_neural} and assuming that $\forall \nu \in \Gc, \log\bra{\nicefrac{\prob{\mu}}{\prob{\nu}}} \in \Span{\Fc}$, we have that with probability at least $1-\delta$ w.r.t. the randomness of training samples,
	\begin{align}
	\tv{g_\alpha}{\mu} \leq \tverr{m,\delta} \triangleq \frac{\sqrt{ \comcoet{\Fc}{\Gc}{\mu} \cdot \ferr{m,\delta}}}{2\sqrt{2}}\;.
	\label{eq:dtv}
	\end{align}
	To show this, note that \cref{thm:error_neural} gives an upper bound on the integral probability metric between the real and generated distribution. We first connect this distance to the KL divergence with the following lemma.
	\begin{lemma}
		Denote the real distribution as $\mu$. Given a generator set $\Gc$ and a discriminator set $\Fc$ which satisfy $\forall \nu \in \Gc, \log\bra{\nicefrac{\prob{\mu}}{\prob{\nu}}} \in \Span{\Fc}$, then we have $\forall g_\alpha \in \Gc$
		\begin{align*}
		\kl{g_\alpha}{\mu} + \kl{\mu}{g_\alpha} \leq \comcoet{\Fc}{\Gc}{\mu} \df{\mu}{g_\alpha} 
		\end{align*}
		where
		$
		\comcoet{\Fc}{\Gc}{\mu} \triangleq \sup_{\nu \in \Gc}\fvnorm{\log\bra{\nicefrac{\prob{\mu}}{\prob{\nu}}}}
		$.
	\end{lemma}
	Similar to \cref{thm:kl_base}, this lemma relies on  Proposition 2.9 in \citet{zhang2017discrimination}. Furthermore, we use Pinsker’s inequality \cite{tsybakov2008introduction} to upper bound the TV distance by the KL distance. Pinsker’s inequality says that $\tv{a}{b} \leq \sqrt{\frac{1}{2}\kl{a}{b}}$ for any two distributions $a,b$. Therefore, we have
	\begin{align*}
	\comcoet{\Fc}{\Gc}{\mu} \df{\mu}{g_\alpha}  \geq 4\cdot \tv{g_\alpha}{\mu}^2
	\end{align*}
	Combing this equation with \cref{thm:error_neural} we get the desired inequality.

	We can use \cref{eq:dtv} to upper bound \cref{eq:tp-bound} as 
	\begin{align}
	\tp{} \leq \fp{} + \min\brc{\tverr{m,\delta}, 1-\fp{}}\;.
	\label{eq:tp-bound-gan}
	\end{align}

Combining \cref{eq:tp-bound-gan} with \cref{thm:roc} gives the result.

%% file: conclusion.tex
\section{Discussion}
\label{sec:conclusion}
In this work, we show that GAN-generated samples naturally exhibit a (weak) differential privacy guarantee as well as protection against membership inference attacks. 
We provide bounds on the privacy risk of each of these attacks. 
However, as discussed in \cref{sec:dp}, the inherent differential privacy guarantee in GANs is weak. 
This suggests that differentially-private training techniques are required to ensure meaningful differential privacy guarantees in practice. 
Given that current techniques like DP-SGD \cite{abadi2016deep} and PATE \cite{jordon2018pate} sacrifice fidelity in exchange for privacy (as discussed in \cref{sec:intro}), 
there is a need to develop new techniques that achieve a better privacy-fidelity tradeoff. 
As noted in \cref{sec:intro}, one approach for achieving this goal may be to ensure the differential privacy for releasing generated samples (rather than enforcing the differential privacy for releasing parameters).

\paragraph{Limitations and future work}\mbox{}

(1) Since our results build on existing generalization bounds for GANs \cite{zhang2017discrimination}, we inherit their assumptions (e.g., on the classes of generators and discriminators). Some of these assumptions do not apply to all GAN architectures. Extending the results to more general settings would be an interesting direction, potentially with stronger generalization bounds. 

(2) As discussed in \cref{sec:membership}, the results on membership inference attacks rely on an assumption from \cite{sablayrolles2019white} regarding the generator parameter distribution; this assumption is unlikely to hold in practice. Relaxing this assumption is an interesting direction for future work.

(3) The bounds we give for both privacy notions  depend on unknown constants like optimization and approximation errors. Numerically quantifying these bounds in practice remains a challenging and interesting direction.

(4) This paper focuses on GANs. Similar techniques could be used to analyze other types of generative models. This could be an interesting direction for future work.

%% file: appendix.tex
\section{Proof of \cref{thm:error_neural_between_g_neighboring}}
\label{app:thm:error_neural_between_g_neighboring}

Assume that the samples of $D_i$ datasets are $x_1^i,...,x_m^i$. Without loss of generality, we assume that $x_i^0=x_i^1$ for $1\leq i \leq m-1$. Then we have
\begin{align*}
\df{\hat{\mu}_m^0}{\hat{\mu}_m^1}
&= \sup_{f\in \Fc}\brc{\Eb_{x\sim \hat{\mu}_m^0}\brb{f(x)} - \Eb_{x\sim \hat{\mu}_m^1}\brb{f(x)}}\\
&=\sup_{f\in \Fc}\brc{ \frac{1}{m} \sum_{i=1}^m f(x_i^0) - \frac{1}{m} \sum_{i=1}^m f(x_i^1)}\\
&= \frac{1}{m} \sup_{f\in \Fc} \brc{ f(x_m^0) - f(x_m^1) }\\
&\leq \frac{2\Delta}{m}
\end{align*}

\section{Proof of \cref{thm:error_neural_between_g}}
\label{app:thm:error_neural_between_g}
From \cref{thm:error_neural}, we know that
\begin{align*}
\Pb\brb{\df{\hat{\mu}_m^0}{g_0} > \frac{1}{2}\ferr{k,\xi}} \leq \xi\\
\Pb\brb{\df{\hat{\mu}_m^1}{g_1} > \frac{1}{2}\ferr{k,\xi}} \leq \xi
\end{align*}
Therefore, 
\begin{align*}
\Pb\brb{\df{\hat{\mu}_m^0}{g_0} \leq  \frac{1}{2}\ferr{k,\xi} \wedge \df{\hat{\mu}_m^1}{g_1} \leq  \frac{1}{2}\ferr{k,\xi}} \geq 1-2\xi.
\end{align*}
With probability at least $1-2\xi$, we have
\begin{align*}
\df{g_0}{g_1}
=& \sup_{f\in \Fc}\brc{\Eb_{x\sim g_0}\brb{f(x)} - \Eb_{x\sim g_1}\brb{f(x)}} \\
=& \sup_{f\in \Fc}\brc{\Eb_{x\sim g_0}\brb{f(x)} - \Eb_{x\sim \hat{\mu}_m^0}\brb{f(x)} 
	+ \Eb_{x\sim \hat{\mu}_m^0}\brb{f(x)} - \Eb_{x\sim \hat{\mu}_m^1}\brb{f(x)}
	+ \Eb_{x\sim \hat{\mu}_m^1}\brb{f(x)} - \Eb_{x\sim g_1}\brb{f(x)}}\\
\leq& \sup_{f\in \Fc}\brc{\Eb_{x\sim g_0}\brb{f(x)} - \Eb_{x\sim \hat{\mu}_m^0}\brb{f(x)}} 
+ \sup_{f\in \Fc}\brc{\Eb_{x\sim \hat{\mu}_m^0}\brb{f(x)} - \Eb_{x\sim \hat{\mu}_m^1}\brb{f(x)}} \\
&\quad+\sup_{f\in\Fc}\brc{\Eb_{x\sim \hat{\mu}_m^1}\brb{f(x)} - \Eb_{x\sim g_1}\brb{f(x)}}\\
=& \sup_{f\in \Fc}\brc{\Eb_{x\sim \hat{\mu}_m^0}\brb{f(x)} - \Eb_{x\sim g_0}\brb{f(x)}} 
+ \sup_{f\in \Fc}\brc{\Eb_{x\sim \hat{\mu}_m^0}\brb{f(x)} - \Eb_{x\sim \hat{\mu}_m^1}\brb{f(x)}} \\
&\quad+\sup_{f\in\Fc}\brc{\Eb_{x\sim \hat{\mu}_m^1}\brb{f(x)} - \Eb_{x\sim g_1}\brb{f(x)}}\\
&~~\text{($\Fc$ is even)}\\
=& \df{\hat{\mu}_m^0}{g_0} + \df{\hat{\mu}_m^0}{\hat{\mu}_m^1} + \df{\hat{\mu}_m^1}{g_1}\\
\leq& \ferr{k,\xi} + \frac{2\Delta}{m}
\end{align*}

\section{Proof of \cref{thm:kl-probdp}}
\label{app:thm:kl-probdp}
Define $\prob{p}$, $\prob{q}$ as the probability (density) functions of $p$ and $q$ respectively. Assume set $S_0=\{x: \log \prob{p}(x)-\log \prob{q}(x)\geq \epsilon\}$, then $\forall x\in S_0$, we have $\prob{p}(x)\geq \prob{q}(x)e^\epsilon$, and
\begin{align*}
s \geq & \kl{p}{q} + \kl{q}{p}\\
=& \int_x \bra{\prob{p}(x)-\prob{q}(x)}\bra{\log \prob{p}(x) - \log \prob{q}(x)}\\
\geq& \int_S \bra{\prob{p}(x)-\prob{q}(x)}\bra{\log \prob{p}(x) - \log \prob{q}(x)}\\
&\text{(because $\bra{\prob{p}(x)-\prob{q}(x)}\bra{\log \prob{p}(x) - \log \prob{q}(x)}\geq 0~~\forall x$)}\\
\geq& \int_S \prob{p}(x) (1-e^{-\epsilon})\epsilon
\end{align*}
i.e. $\Pb\brb{M(D_0)\in S_0}\leq \frac{s}{\epsilon (1-e^{-\epsilon})}$. For any set $S$, we have 
\begin{align*}
\Pb\brb{M(D_0)\in S\setminus S_0}
=& \int_{S\setminus S_0} \prob{p}(x)dx \\
\leq& \int_{S\setminus S_0} \prob{q}(x)e^\epsilon dx\\
=& e^\epsilon\Pb\brb{M(D_1)\in S\setminus S_0}
\end{align*}

\section{Proof of \cref{thm:error_neural_between_g_upper}}
\label{app:thm:error_neural_between_g_upper}
Recall that in \cref{app:thm:error_neural_between_g_upper} we get
\begin{align*}
\Pb\brb{\df{\hat{\mu}_m^0}{g_0} \leq  \frac{1}{2}\ferr{k,\xi} \wedge \df{\hat{\mu}_m^1}{g_1} \leq  \frac{1}{2}\ferr{k,\xi}} \geq 1-2\xi.
\end{align*}
With probability at least $1-2\xi$, we have
\begin{align*}
\df{g_0}{g_1}
=& \sup_{f\in \Fc}\brc{\Eb_{x\sim g_0}\brb{f(x)} - \Eb_{x\sim g_1}\brb{f(x)}} \\
=& \sup_{f\in \Fc}\brc{\Eb_{x\sim g_0}\brb{f(x)} - \Eb_{x\sim \hat{\mu}_m^0}\brb{f(x)} 
	+ \Eb_{x\sim \hat{\mu}_m^0}\brb{f(x)} - \Eb_{x\sim \hat{\mu}_m^1}\brb{f(x)}
	+ \Eb_{x\sim \hat{\mu}_m^1}\brb{f(x)} - \Eb_{x\sim g_1}\brb{f(x)}}\\
\geq& -\sup_{f\in \Fc}\brc{\Eb_{x\sim g_0}\brb{f(x)} - \Eb_{x\sim \hat{\mu}_m^0}\brb{f(x)}} 
+ \sup_{f\in \Fc}\brc{\Eb_{x\sim \hat{\mu}_m^0}\brb{f(x)} - \Eb_{x\sim \hat{\mu}_m^1}\brb{f(x)}} \\
&\quad-\sup_{f\in\Fc}\brc{\Eb_{x\sim \hat{\mu}_m^1}\brb{f(x)} - \Eb_{x\sim g_1}\brb{f(x)}}\\
=& -\df{\hat{\mu}_m^0}{g_0} + \df{\hat{\mu}_m^0}{\hat{\mu}_m^1} - \df{\hat{\mu}_m^1}{g_1}\\
\geq&  \frac{\Delta'}{2m} - \ferr{k,\xi}
\end{align*}

\section{Proof of \cref{thm:dp_tv}}
Assume that $S=\brc{x\in X|p(x)>q(x)}$ and $T=X\setminus S=\brc{x\in X|p(x)<=q(x)}$. Let $a_1=\int_{x\in S} p(x)dx$, $b_1=\int_{x\in S}q(x)dx$, $a_2=\int_{x\in T}p(x)dx$, and $b_2=\int_{x\in T}q(x)dx$. Because of the the differential privacy guarantee, we have
\begin{align*}
	a_1-\delta \leq e^\epsilon b_1\\
	b_2 - \delta \leq e^\epsilon a_2
\end{align*}
Note that $a_1+a_2=1$, $b_1+b_2=1$.
Therefore,  we have
\begin{align*}
	b_1+a_2 \geq \frac{2-2\delta}{1+e^\epsilon}
\end{align*}
and
\begin{align*}
	\tv{p}{q} = \frac{a_1+b_2-b_1-a_2}{2} \leq \frac{e^\epsilon+2\delta-1}{e^\epsilon+1}.
\end{align*}